\pdfoutput=1
\documentclass[11pt]{article}
\usepackage[margin=1in]{geometry}
\usepackage[T1]{fontenc}
\usepackage[utf8]{inputenc}
\usepackage{lmodern}
\usepackage{amsmath,amssymb,amsthm,mathtools,bm}
\usepackage{enumitem}
\usepackage{booktabs}
\usepackage{hyperref}
\usepackage[numbers,sort&compress]{natbib}
\usepackage{xcolor}
\hypersetup{colorlinks=true,linkcolor=blue,citecolor=blue,urlcolor=blue}

\title{Artificial Jagged Intelligence as Uneven Optimization Energy Allocation\\
\large Capability Concentration, Redistribution, and Optimization Governance}
\author{%
Wesley Shu\\
The Institute of Energetic Paradigm\\
\texttt{shu@energeticparadigm.org}
\and
Peng WEI\\
College of Plant Protection\\
Southwest University, China\\
\texttt{weipeng2019@swu.edu.cn}
}
\date{}

\newtheorem{assumption}{Assumption}
\newtheorem{proposition}{Proposition}
\newtheorem{theorem}{Theorem}
\newtheorem{corollary}{Corollary}

\begin{document}
\maketitle

\begin{abstract}
Artificial Jagged Intelligence (AJI) denotes a recurring pattern in which large learning systems exhibit strong local capabilities while remaining weak or brittle in other domains. This paper develops a formal theory of AJI as uneven allocation of optimization pressure. We model training as a finite-budget process that distributes gradient-driven update energy across capability-relevant directions in parameter space. In this model, jagged capability profiles arise from anisotropic objective structure, data geometry, and representational coupling rather than from a single scalar quantity called intelligence.

The paper defines capability gain, optimization energy share, and jaggedness, then proves that persistent concentration of cumulative update energy yields lower bounds on dispersion in capability gains. A finite-budget tradeoff theorem shows why prioritizing one capability can impose opportunity costs on others unless positive coupling or shared structure offsets the cost. The analysis also studies redistribution mechanisms, including energy-variance regularization and auxiliary structural objectives, as interventions that reshape the optimization field.

The resulting framework links uneven emergence, training architecture, and optimization governance. It predicts that early concentration of update energy should forecast later capability jaggedness; that scaling under a narrow objective need not eliminate anisotropy; and that explicitly funded auxiliary objectives can revive neglected capabilities. AJI is therefore not merely a descriptive label for uneven model behavior, but a testable theory of how finite optimization resources produce concentrated, delayed, and structurally uneven capability formation.
\end{abstract}

\section{Introduction}
Large AI systems rarely improve in a smooth, globally uniform manner. The empirical record instead shows clustered strength: models may become highly effective at code synthesis, algebraic manipulation, or local reasoning while remaining brittle on long-horizon planning, cross-context consistency, calibrated uncertainty, or stable value-sensitive behavior. This uneven profile is visible across work on scaling laws, emergent abilities, grokking, sparse expert routing, and mechanistic interpretability \citep{kaplan2020scaling,hoffmann2022training,wei2022emergent,schaeffer2023mirage,power2022grokking,olsson2022induction,elhage2022superposition,ganguli2022predictability}.

This paper argues that such unevenness is not well captured by a single scalar notion of intelligence. The more useful explanatory object is the allocation of optimization pressure. Gradient-based training distributes finite update resources across parameter-space directions that differ in their effect on the objective. Directions that reduce the loss efficiently attract repeated update flow; directions that are weakly represented in the objective, sparsely supported by data, or negatively coupled to dominant features receive comparatively little. On this view, capability jaggedness is a structural consequence of anisotropic optimization.

Our contribution is fivefold. First, we define a capability-space model in which behaviorally meaningful capabilities are represented by differentiable functionals and their development is driven by projected gradient energy. Second, we formalize jaggedness and prove lower bounds connecting concentration of cumulative update energy to dispersion in capability gains. Third, we derive a finite-budget tradeoff theorem and a coupling proposition showing how representational interaction changes the cost of balancing capabilities. Fourth, we analyze redistribution mechanisms showing how regularization and auxiliary losses alter the induced allocation. Fifth, we interpret optimization constraints as a governable control layer over model development.

The paper is primarily theoretical. Its aim is not to collapse all of learning into an energy metaphor, but to isolate an operational primitive: under finite compute, data, and training time, learning systems allocate update pressure unevenly. Once stated in those terms, several recurring phenomena become easier to organize: capability spikes, neglected capacities, interference across objectives, and the possibility of governing developmental direction through constraints on the optimization field.

The paper proceeds as follows. Section 2 situates the framework relative to scaling, emergence, multi-objective optimization, interpretability, and governance. Sections 3--7 develop the formal model and its main propositions. Sections 8--11 extend the theory to governance, identifiability, and empirical predictions. The final sections discuss limitations, design implications, and the broader research program.

\section{Related Work and Positioning}
The present framework intersects several active literatures but makes a different explanatory move: it treats uneven capability formation itself as the primary object of theory.

\paragraph{Scaling, emergence, and threshold behavior.} Scaling-law work shows that language-model loss and many downstream metrics often follow regular statistical trends as model size, data, and compute increase \citep{kaplan2020scaling,hoffmann2022training,bahri2024explaining}. A related literature documents abrupt or threshold-like gains in particular tasks at scale \citep{wei2022emergent,ganguli2022predictability,bubeck2023sparks}. Subsequent critiques argue that some apparently discontinuous gains may instead arise from evaluation thresholds, nonlinear metrics, or aggregation choices \citep{schaeffer2023mirage}. The present account is compatible with both descriptive pictures. Even when latent learning dynamics are smooth, anisotropic update allocation can generate sharp observed transitions once a capability crosses task-specific thresholds.

\paragraph{Optimization geometry and delayed generalization.} Grokking and related studies show that long periods of mediocre performance can conceal substantial representational reorganization, later expressed as sudden generalization \citep{power2022grokking,nanda2023progress,predgrok2023}. Neural tangent kernel and lazy-training analyses clarify regimes in which models behave approximately linearly in parameters \citep{jacot2018ntk,chizat2019lazy}, while representation-learning perspectives stress departures from that regime as feature learning becomes dominant \citep{fort2020deep,neyshabur2018role}. These literatures motivate a focus on how update flow is distributed over training, not merely on endpoint capacity.

\paragraph{Multi-task and multi-objective optimization.} Methods such as MGDA, GradNorm, PCGrad, IMTL, CAGrad, Nash-MTL, and related balancing procedures were developed because gradients associated with different tasks may conflict, dominate, or underrepresent one another \citep{sener2018multi,chen2018gradnorm,yu2020pcgrad,liu2021cagrad,navon2022nash,du2018adapting,javaloy2022rotograd}. Our framework extends this logic from explicit task heads to capability-relevant directions in a shared model. The central question becomes not only how to combine per-task gradients, but how finite optimization pressure is structurally distributed across behaviorally meaningful dimensions.

\paragraph{Mechanistic interpretability and selective specialization.} Circuit analyses, induction-head work, superposition toy models, sparse autoencoders, and privileged-basis results suggest that internal structures emerge selectively and may share or contest representational resources \citep{olah2021framework,olsson2022induction,elhage2022superposition,privileged2023,bricken2023sae,anthropic2024features}. This is congenial to an allocation-based view: capability formation depends not only on architecture size, but on where structured update pressure accumulates and which internal mechanisms are repeatedly reinforced.

\paragraph{Governance and control of frontier systems.} Recent governance work has focused on evaluations, model release, safety cases, data governance, and broader public-safety regulation \citep{bommasani2021foundation,amodei2016concrete,hubinger2019risks,shevlane2023modeleval,anderljung2023frontier,hausenloy2024data}. Our contribution intersects with that literature by introducing a training-internal object of governance: the optimization field itself. On this view, governance can act not only on deployment or output restrictions, but also on the developmental trajectory induced by the allocation of update energy during training.

\section{Formal Setup}
Let parameter space be $\Theta \subseteq \mathbb{R}^d$, with model parameters $\theta \in \Theta$. Training minimizes a differentiable loss
\[
L(\theta) = \mathbb{E}_{z\sim \mathcal{D}}[\ell(\theta; z)].
\]
Let there be $m$ capability coordinates indexed by $i \in \{1,\dots,m\}$. Each capability is represented by a differentiable functional
\[
C_i : \Theta \to \mathbb{R},
\]
where larger $C_i(\theta)$ indicates stronger performance on capability dimension $i$.

The gradient update is
\[
\theta_{t+1} = \theta_t - \eta_t g_t,
\qquad g_t := \nabla L(\theta_t).
\]
For each capability coordinate, define the instantaneous capability gain under the update by first-order approximation:
\[
\Delta C_i(t) \approx -\eta_t \langle \nabla C_i(\theta_t), g_t \rangle.
\]
Thus capability growth depends on the projection of the loss gradient onto the sensitivity direction of capability $i$.

\subsection{Energy Allocation}
Define the optimization energy allocated to capability $i$ at step $t$ as
\[
E_i(t) := \frac{|\langle \nabla C_i(\theta_t), g_t \rangle|}{\sum_{j=1}^m |\langle \nabla C_j(\theta_t), g_t \rangle|}.
\]
Then $E_i(t) \in [0,1]$ and $\sum_i E_i(t)=1$. This measures the share of first-order update pressure incident on capability coordinate $i$.

For a training horizon $T$, define cumulative energy share
\[
\bar E_i(T) := \frac{\sum_{t=0}^{T-1} \eta_t |\langle \nabla C_i(\theta_t), g_t \rangle|}{\sum_{j=1}^m \sum_{t=0}^{T-1} \eta_t |\langle \nabla C_j(\theta_t), g_t \rangle|}.
\]

\subsection{Jaggedness}
Let the cumulative capability gain over horizon $T$ be
\[
G_i(T) := C_i(\theta_T) - C_i(\theta_0).
\]
We define \emph{jaggedness} by normalized dispersion of capability gains:
\[
J(T) := \frac{1}{m}\sum_{i=1}^m \left( G_i(T) - \frac{1}{m}\sum_{k=1}^m G_k(T) \right)^2.
\]
Large $J(T)$ indicates uneven capability development.

\section{Assumptions}
We work under the following assumptions.

\begin{assumption}[Smoothness]
Each $C_i$ is differentiable with $L_i$-Lipschitz gradient and $L$ is differentiable.
\end{assumption}

\begin{assumption}[Positive alignment regime]
Along the training trajectory over horizon $T$, the signed projections satisfy
\[
\langle \nabla C_i(\theta_t), -g_t \rangle \ge 0
\]
for the relevant capabilities under study.
\end{assumption}

\begin{assumption}[Bounded sensitivity]
There exist constants $0<\alpha_i\le \beta_i<\infty$ such that
\[
\alpha_i \|g_t\| E_i(t) \le \langle \nabla C_i(\theta_t), -g_t \rangle \le \beta_i \|g_t\| E_i(t).
\]
This states that projected update energy controls first-order capability growth up to capability-specific constants.
\end{assumption}

The bounded-sensitivity assumption can be weakened, but it makes the relation between allocation and gain analytically transparent.

\section{Capability Concentration Implies Jaggedness}
We now establish the core result: concentrated optimization energy produces uneven capability growth.

\begin{proposition}[Energy-to-gain comparison]
Under Assumptions 1--3,
\[
\sum_{t=0}^{T-1} \eta_t \alpha_i \|g_t\| E_i(t) + R_i^-(T)
\le G_i(T) \le
\sum_{t=0}^{T-1} \eta_t \beta_i \|g_t\| E_i(t) + R_i^+(T),
\]
where the remainder terms satisfy
\[
|R_i^{\pm}(T)| \le \frac{L_i}{2} \sum_{t=0}^{T-1} \eta_t^2 \|g_t\|^2.
\]
\end{proposition}

\begin{proof}
By smoothness of $C_i$,
\[
C_i(\theta_{t+1}) \ge C_i(\theta_t) + \langle \nabla C_i(\theta_t), \theta_{t+1}-\theta_t \rangle - \frac{L_i}{2}\|\theta_{t+1}-\theta_t\|^2.
\]
Substituting $\theta_{t+1}-\theta_t=-\eta_t g_t$ gives
\[
C_i(\theta_{t+1}) - C_i(\theta_t)
\ge \eta_t \langle \nabla C_i(\theta_t), -g_t \rangle - \frac{L_i}{2}\eta_t^2\|g_t\|^2.
\]
Applying Assumption 3 yields the lower bound. The upper bound follows analogously from the smoothness upper inequality. Summing over $t=0,\dots,T-1$ proves the claim.
\end{proof}

\begin{theorem}[Concentration lower bound on jaggedness]\label{thm:jaggedness}
Assume Assumptions 1--3 and for simplicity let $\alpha_i=\alpha$, $\beta_i=\beta$ for all $i$. Define
\[
W_i(T):=\sum_{t=0}^{T-1}\eta_t\|g_t\|E_i(t),
\qquad
\bar W(T):=\frac{1}{m}\sum_{i=1}^m W_i(T).
\]
Then there exists a constant $K_T\ge 0$ with
\[
J(T) \ge \alpha^2 \cdot \frac{1}{m}\sum_{i=1}^m (W_i(T)-\bar W(T))^2 - K_T,
\]
where
\[
K_T = O\!\left(\left(\sum_{t=0}^{T-1}\eta_t^2\|g_t\|^2\right)^2\right).
\]
In particular, persistent dispersion in cumulative energy allocation induces nontrivial capability jaggedness.
\end{theorem}

\begin{proof}
By Proposition 1, write
\[
G_i(T) = \tilde G_i(T) + R_i(T),
\qquad
\tilde G_i(T) \ge \alpha W_i(T),
\]
with $|R_i(T)| \le c_T := \frac{L_*}{2}\sum_t \eta_t^2\|g_t\|^2$, where $L_*:=\max_i L_i$. Let $\bar G=(1/m)\sum_i G_i$ and $\bar R=(1/m)\sum_i R_i$. Then
\[
G_i-\bar G = (\tilde G_i-\bar{\tilde G}) + (R_i-\bar R).
\]
Using $(a+b)^2 \ge a^2 -2|a||b|$ and summing over $i$,
\[
J(T) \ge \frac{1}{m}\sum_i (\tilde G_i-\bar{\tilde G})^2 - \frac{2}{m}\sum_i |\tilde G_i-\bar{\tilde G}|\,|R_i-\bar R|.
\]
The first term is bounded below by
\[
\alpha^2 \frac{1}{m}\sum_i (W_i-\bar W)^2.
\]
The second term is controlled by Cauchy--Schwarz and the uniform remainder bound, yielding an error term of order $c_T^2$. Hence
\[
J(T) \ge \alpha^2 \frac{1}{m}\sum_i (W_i-\bar W)^2 - K_T.
\]
\end{proof}

\begin{corollary}
If one capability $i^*$ receives asymptotically dominant cumulative energy share, i.e.
\[
\bar E_{i^*}(T) \to 1,
\]
while at least one other capability has $\bar E_j(T) \to 0$, then jaggedness remains bounded away from zero up to second-order training errors.
\end{corollary}

\section{Objective Mismatch and Neglected Capabilities}
We next model why some capabilities are systematically undertrained. Suppose the task loss decomposes as
\[
L(\theta)=L_{\mathrm{prox}}(\theta)+\varepsilon L_{\mathrm{struct}}(\theta),
\qquad 0\le \varepsilon \ll 1,
\]
where $L_{\mathrm{prox}}$ rewards short-horizon predictive performance and $L_{\mathrm{struct}}$ rewards long-horizon or structurally integrated behavior.

If capability $C_i$ is primarily aligned with $-\nabla L_{\mathrm{struct}}$ but weakly aligned with $-\nabla L_{\mathrm{prox}}$, then the first-order gain for $C_i$ is suppressed by the small coefficient $\varepsilon$. This formalizes objective mismatch: capabilities weakly represented in the dominant loss receive weak gradient energy.

\begin{proposition}[Underinvestment under objective mismatch]
Assume
\[
\langle \nabla C_i, -\nabla L_{\mathrm{prox}} \rangle \le \delta,
\qquad
\langle \nabla C_i, -\nabla L_{\mathrm{struct}} \rangle \le M
\]
with $\delta$ small. Then for gradient descent on $L$,
\[
\Delta C_i(t) \lesssim \eta_t(\delta + \varepsilon M),
\]
up to second-order terms. Hence when $\varepsilon$ and $\delta$ are small, capability $i$ remains systematically under-optimized.
\end{proposition}

\begin{proof}
Immediate from linearity of the gradient:
\[
\langle \nabla C_i, -\nabla L \rangle
= \langle \nabla C_i, -\nabla L_{\mathrm{prox}} \rangle + \varepsilon \langle \nabla C_i, -\nabla L_{\mathrm{struct}} \rangle
\le \delta + \varepsilon M.
\]
Substitute into the first-order capability gain formula.
\end{proof}

\section{Redistribution Mechanisms}
We now examine two intervention families: energy-balancing regularizers and auxiliary structural objectives.

\subsection{Energy Variance Penalty}
Suppose training minimizes
\[
L_{\mathrm{tot}}(\theta)=L(\theta)+\lambda \operatorname{Var}(E_1(\theta),\dots,E_m(\theta)).
\]
This penalizes concentration of optimization energy across capabilities or modules.

\begin{proposition}[Regularization reduces concentration incentives]
Under differentiability of the allocation map $\theta \mapsto E(\theta)$, any stationary point of $L_{\mathrm{tot}}$ must satisfy
\[
\nabla L(\theta) = -\lambda \nabla \operatorname{Var}(E(\theta)).
\]
Thus highly concentrated energy profiles can persist only when offset by commensurate task-loss advantage. Increasing $\lambda$ reduces the set of admissible concentrated stationary points.
\end{proposition}

\begin{proof}
This is the first-order stationary condition for the composite objective. As $\lambda$ increases, the regularizer contributes greater opposing force against directions that increase variance in $E$.
\end{proof}

\subsection{Auxiliary Capability Objectives}
Let
\[
L_{\mathrm{tot}} = L_{\mathrm{token}} + \sum_{i\in \mathcal{S}} \gamma_i L_i^{\mathrm{aux}},
\]
where $\mathcal{S}$ indexes structurally neglected capabilities.

\begin{proposition}[Auxiliary objectives lift neglected energy shares]
Assume for some neglected capability $k$ that
\[
\langle \nabla C_k, -\nabla L_{\mathrm{token}} \rangle \approx 0,
\quad
\langle \nabla C_k, -\nabla L_k^{\mathrm{aux}} \rangle > 0.
\]
Then for sufficiently large $\gamma_k>0$, the induced energy share $E_k(t)$ becomes strictly positive along the trajectory, up to higher-order coupling effects.
\end{proposition}

\begin{proof}
The combined gradient is
\[
-\nabla L_{\mathrm{tot}} = -\nabla L_{\mathrm{token}} - \sum_{i\in \mathcal S} \gamma_i \nabla L_i^{\mathrm{aux}}.
\]
Projecting onto $\nabla C_k$ yields a term proportional to $\gamma_k\langle \nabla C_k,-\nabla L_k^{\mathrm{aux}}\rangle$, which dominates once $\gamma_k$ exceeds the coupling and remainder terms.
\end{proof}

\subsection{A Finite-Budget Tradeoff Theorem}
The core energetic claim is not merely that capabilities may develop unevenly, but that under finite training resources, strengthening one region of capability space can impose opportunity costs on others.

Let the training budget over horizon $T$ be summarized by
\[
B_T := \sum_{t=0}^{T-1} \eta_t \|g_t\|.
\]
Since $\sum_i E_i(t)=1$, we have
\[
\sum_{i=1}^m W_i(T)=B_T,
\qquad
W_i(T)=\sum_{t=0}^{T-1}\eta_t\|g_t\|E_i(t).
\]
Thus cumulative energy weights are constrained to lie on a simplex with radius determined by the finite budget.

\begin{theorem}[Finite-budget capability tradeoff]\label{thm:budget}
Under Assumptions 1--3 with common lower sensitivity constant $\alpha>0$, the cumulative gains satisfy
\[
\sum_{i=1}^m G_i(T) \le \beta B_T + O\!\left(\sum_{t=0}^{T-1}\eta_t^2\|g_t\|^2\right),
\]
while each individual gain obeys
\[
G_i(T) \ge \alpha W_i(T) - O\!\left(\sum_{t=0}^{T-1}\eta_t^2\|g_t\|^2\right).
\]
Consequently, if one capability's cumulative weight $W_{i^*}(T)$ is increased while $B_T$ is fixed, then, absent higher-order synergy terms, the average feasible lower bound for the remaining capabilities must decrease.
\end{theorem}

\begin{proof}
The upper bound on total gain follows from Proposition 1 and the identity $\sum_i W_i(T)=B_T$, together with the common upper sensitivity constant $\beta$. The lower bound is immediate from Proposition 1. Since the total budget is fixed, increasing one $W_{i^*}(T)$ reduces the residual budget available to the remaining coordinates. Therefore the sum of their lower bounds declines by at least $\alpha\,\Delta W_{i^*}(T)$ up to second-order errors.
\end{proof}

\subsection{Coupling and Synergy}
The finite-budget theorem does not imply that all capability growth is strictly zero-sum. In realistic systems, capabilities may share representational primitives, so training energy directed toward one coordinate may partially benefit others.

Define pairwise coupling coefficients
\[
\kappa_{ij}(t) := \frac{\langle \nabla C_i(\theta_t), \nabla C_j(\theta_t) \rangle}{\|\nabla C_i(\theta_t)\|\,\|\nabla C_j(\theta_t)\|}.
\]
Positive coupling indicates aligned local sensitivity; negative coupling indicates interference.

\begin{proposition}[Coupling-modified gain]
Suppose the loss gradient decomposes as
\[
g_t=\sum_j a_j(t)\nabla C_j(\theta_t)+r_t,
\]
where the residual $r_t$ is orthogonal to the capability span. Then
\[
\Delta C_i(t) \approx -\eta_t \sum_j a_j(t)\|\nabla C_i(\theta_t)\|\,\|\nabla C_j(\theta_t)\|\kappa_{ij}(t).
\]
Hence energy directed toward capability $j$ helps capability $i$ when $\kappa_{ij}(t)>0$ and harms it when $\kappa_{ij}(t)<0$.
\end{proposition}

\begin{proof}
Substitute the decomposition of $g_t$ into the first-order gain formula and use orthogonality of $r_t$.
\end{proof}

This proposition refines the budget argument: energy allocation remains scarce, but its effects are filtered through a coupling structure. Jaggedness is therefore strongest when allocation is concentrated and couplings are weak or negative. This is one way to connect optimization concentration to interference and gradient conflict phenomena observed in multitask learning \citep{sener2018multi,yu2020pcgrad,liu2021cagrad}.

\section{Optimization Governance as Constrained Control}
The previous sections suggest a governance interpretation. Let the training process be viewed as a controlled dynamical system
\[
\theta_{t+1}=\theta_t-\eta_t U_t g_t,
\]
where $U_t$ is a diagonal or block-diagonal control operator reweighting update flow across modules or capability-relevant subspaces.

A governance layer imposes admissible constraints such as
\[
E_i(t) \le \rho_i^{\max},
\qquad
E_j(t) \ge \rho_j^{\min}
\]
for designated dimensions or modules. This converts capability formation from a purely emergent byproduct into a partially steered process. The technical question becomes not whether capability growth can occur, but which regions of capability space are preferentially supplied with update energy under institutional constraints.

This governance view differs from post hoc evaluation or deployment restriction. It acts on the developmental mechanism itself. Once gradient-energy distribution becomes measurable, organizations can reserve update quotas for underrepresented capabilities, cap concentration in strategically dangerous directions, or stage capability maturation across training phases. In that sense, optimization governance resembles constrained control more than external compliance rhetoric \citep{bommasani2021foundation,shevlane2023modeleval,anderljung2023frontier}.

\section{Identifiability and Measurement}
The formal framework requires capability coordinates $C_i$, but these are not unique. Several measurement strategies are available.

\paragraph{Task-defined coordinates.} One may define $C_i$ through benchmark-specific performance functionals, for example long-horizon planning success, consistency across contexts, or calibration. This is operationally simple but sensitive to benchmark construction.

\paragraph{Module-defined coordinates.} Alternatively, one may partition parameters or internal activations by modules, experts, layers, or discovered circuits and define energy shares on those modules. This approach is compatible with mixture-of-experts models, sparse routing, and interpretability workflows \citep{shazeer2017moe,lepikhin2020gshard,fedus2021switch}.

\paragraph{Representation-defined coordinates.} A more ambitious approach uses latent features discovered by sparse autoencoders or circuit extraction. Then $C_i$ need not correspond to human-labeled tasks; it can represent an internal behavioral primitive \citep{bricken2023sae,anthropic2024features}.

None of these strategies is canonical. The theory should therefore be understood as coordinate-dependent but not vacuous. Many physical theories require a choice of state variables; the question is whether the chosen variables support explanation, intervention, and prediction.

\section{Empirical Predictions and Falsifiability}
A theory paper is stronger when it generates observable consequences. The present framework yields at least five testable predictions.

\begin{enumerate}[leftmargin=2em]
\item \textbf{Early concentration predicts later jaggedness.} If module-level or capability-projected gradient energy can be estimated during training, early dispersion in cumulative energy shares should predict later dispersion in capability gains better than parameter count alone.

\item \textbf{Scaling without objective diversification preserves anisotropy.} Increasing model size or data under a fixed narrow objective should improve absolute performance broadly while leaving a nontrivial floor on normalized jaggedness.

\item \textbf{Auxiliary objectives revive neglected capabilities.} Injecting explicit losses for planning, consistency, or long-horizon coherence should increase the measured energy share of those dimensions and accelerate their development relative to a token-only baseline.

\item \textbf{Energy-balancing regularization reduces spikes at some performance cost.} Penalizing concentration should decrease gain dispersion and improve undertrained capabilities, while possibly sacrificing peak efficiency on the most directly rewarded dimensions.

\item \textbf{Coupling structure moderates the tradeoff.} Positive inter-capability coupling should make redistribution cheaper, whereas antagonistic coupling should make balancing more costly and more necessary.
\end{enumerate}

These predictions can be studied in synthetic modular tasks, multitask benchmarks, sparse-expert architectures, and instrumented large-model training runs. They also connect naturally to work on evaluations, safety cases, and extreme-risk model assessment \citep{shevlane2023modeleval,perez2022redteam}.

\section{A Toy Linear Example}
To make the allocation mechanism concrete, consider a simplified linear model with parameter vector $\theta\in\mathbb{R}^d$ and two capability coordinates
\[
C_1(\theta)=u^\top \theta, \qquad C_2(\theta)=v^\top \theta,
\]
for unit vectors $u,v\in\mathbb{R}^d$. Let the training loss be quadratic,
\[
L(\theta)=\frac12\|A\theta-b\|^2,
\]
with gradient $g=A^\top(A\theta-b)$. Suppose first that $u$ is nearly aligned with a dominant right-singular direction of $A$ while $v$ is nearly orthogonal to that direction. Then standard gradient descent sends larger projected updates into $u$ than into $v$. The energy shares become
\[
E_1(t) \approx \frac{|u^\top g_t|}{|u^\top g_t|+|v^\top g_t|},
\qquad
E_2(t) \approx \frac{|v^\top g_t|}{|u^\top g_t|+|v^\top g_t|}.
\]
If the spectrum of $A$ is anisotropic, so that $|u^\top g_t|\gg |v^\top g_t|$ over a substantial prefix of training, then Theorem~\ref{thm:jaggedness} predicts persistent jaggedness. The model improves rapidly along $C_1$ and slowly along $C_2$ even though both are linear observables over the same underlying parameters.

This toy example is intentionally minimal, but it highlights an important point. Uneven capability growth does not require exotic architecture, discrete phase transitions, or mysterious emergence. It can arise whenever the optimization field itself is anisotropic. Nonlinear deep models intensify this effect because their capability coordinates are not fixed in advance; they are jointly learned with the representation. Nevertheless, the linear example already shows why some directions become low-impedance update channels while others remain weakly funded.

A natural extension introduces an auxiliary objective for the second capability:
\[
L_{\mathrm{tot}}(\theta)=\frac12\|A\theta-b\|^2+\gamma \,\frac12 (v^\top\theta-c)^2.
\]
Then the gradient acquires an additional component proportional to $(v^\top\theta-c)v$, which directly increases the projected energy entering capability $2$. In this sense, auxiliary objectives do not merely evaluate missing capabilities; they change the vector field of optimization itself.

\section{Limits, Extensions, and Research Program}
The present theory can be extended in several directions.

First, it can be generalized from first-order gradient descent to stochastic optimization with momentum, adaptive preconditioning, or second-order information. This matters because optimizers such as Adam may alter effective energy distribution through coordinate-wise normalization, which is relevant to observations about privileged residual-stream bases and optimizer-induced anisotropy \citep{privileged2023}. A fuller theory would model not only raw gradient flow but optimizer-transformed flow.

Second, the current framework treats capabilities as scalar coordinates. In practice, many capabilities are hierarchical bundles: planning depends on memory, state abstraction, credit assignment, and search; calibration depends on representation, decoding, and distributional alignment. A more realistic approach would model capability manifolds or directed dependency graphs. Jaggedness would then reflect not only dispersion across coordinates but structural bottlenecks inside capability bundles.

Third, the theory suggests an empirical program. One can instrument model training to estimate module-level or feature-level gradient-energy shares, then test whether early allocation imbalance predicts later behavioral jaggedness. One can vary objective mixtures, optimizer choice, sparse routing, or regularization strength and measure how the capability profile changes. Even negative results would be informative, because they would indicate that projected update allocation is not the right explanatory layer or that the chosen coordinates fail to capture real developmental structure.

Finally, the governance implications can be sharpened. If training-internal allocation proves measurable and controllable, then institutional actors may have leverage over developmental trajectories long before deployment. This would move a portion of AI governance from outer-loop monitoring to inner-loop design. Such a shift would not eliminate risk, but it would change where control can be exercised and what kinds of evidence count as responsible development.

\section{Design Implications for Training Systems}
The formal results imply a practical design distinction between \emph{capability maximization} and \emph{capability shaping}. Capability maximization asks how to obtain the strongest possible performance under a fixed objective. Capability shaping asks how to allocate update pressure so that the resulting system develops a desired capability profile rather than simply the most loss-reducing one.

This distinction matters because many safety-relevant or governance-relevant properties are not naturally privileged by narrow predictive losses. Long-horizon consistency, calibrated refusal, robust uncertainty estimation, and reversible human oversight may all be weakly represented under token-level training. If so, then simply scaling data and parameters will not reliably mature those properties at the same rate as local next-token competence. From the standpoint of optimization, such properties are underfunded unless their gradients are deliberately amplified or protected.

A training architecture based on the present framework would therefore include at least four engineering components: (i) continuous monitoring of energy concentration across modules or capability coordinates; (ii) redistribution mechanisms that cap extreme concentration or raise floors for neglected dimensions; (iii) auxiliary objectives that create direct gradients for underrepresented capabilities; and (iv) reporting or auditing interfaces that expose how developmental budget was allocated over training time. The conceptual shift is from measuring only output quality to also measuring the internal distribution of developmental investment.

\section{Discussion}
The framework developed here offers a structural explanation for uneven capability emergence without relying on a single scalar measure of intelligence. It predicts that scaling alone need not eliminate jaggedness when the objective and gradient geometry remain anisotropic; that some persistent capability gaps may be products of systematic underallocation rather than intrinsic impossibility; and that training-time governance is technically meaningful because developmental direction depends on how update pressure is routed.

The theory also has limits. Capability coordinates are not uniquely identifiable, projected energy is measurement-dependent, and local first-order gains may miss delayed, compositional, or highly nonlinear forms of structure formation. The framework should therefore be read as a tractable approximation rather than a total theory of learning. Its value lies in organizing a family of empirical regularities and design questions around a shared operational primitive.

\section{Conclusion}
Artificial Jagged Intelligence can be understood as a consequence of uneven optimization energy allocation. By formalizing capability gain, energy share, jaggedness, tradeoff, and coupling, this paper turns an intuitive observation about uneven model competence into a structured theory of capability formation. The central implication is that the developmental shape of advanced AI systems is not determined by scale alone. It is determined by the objective structure, data geometry, coupling structure, and governance constraints that decide where finite update pressure accumulates.

That claim does not eliminate the need for empirical work. It sharpens it. The next step for this research program is to instrument real training systems, estimate capability-relevant energy shares, and test whether early allocation patterns predict later jaggedness better than scale variables alone. If they do, then optimization governance is not merely a policy metaphor but a concrete design layer in the development of advanced AI systems.

\end{document}